\documentclass[10pt,twocolumn,letterpaper]{article}

\usepackage{cvpr}
\usepackage{times}
\usepackage{epsfig}
\usepackage{graphicx}
\usepackage{amsmath}
\usepackage{amssymb}
\usepackage{amsthm}
\usepackage{subcaption}
\usepackage[ruled]{algorithm2e}
\usepackage{algpseudocode}
\newcommand\blfootnote[1]{%
  \begingroup
  \renewcommand\thefootnote{}\footnote{#1}%
  \addtocounter{footnote}{-1}%
  \endgroup
}

\usepackage{mathtools}
\DeclarePairedDelimiter\ceil{\lceil}{\rceil}

\newtheorem{thm}{Theorem}[]

\newtheorem{assump}[thm]{Assumption}

\SetKwInput{kwEvaluate}{Evaluate}
\SetKwInput{kwSort}{Sort}
\SetKwInput{kwInput}{Input}
\SetKwInput{kwOutput}{Output}
\SetKwInput{kwInitialize}{Initialize}
\SetKwInput{kwParameters}{Params.}
\SetKwInput{kwDefaults}{Default init.}

\usepackage[pagebackref=true,breaklinks=true,letterpaper=true,colorlinks,bookmarks=false]{hyperref}

\cvprfinalcopy 

\def\cvprPaperID{10044} 

\ifcvprfinal\fi
\begin{document}

\title{Correction Filter for Single Image Super-Resolution: \\ Robustifying Off-the-Shelf Deep Super-Resolvers}


\author{Shady Abu Hussein\\
Tel Aviv University, Israel\\
{\tt\small shadya@mail.tau.ac.il}
\and
Tom Tirer\\
Tel Aviv University, Israel\\
{\tt\small tomtirer@mail.tau.ac.il}
\and
Raja Giryes\\
Tel Aviv University, Israel\\
{\tt\small raja@tauex.tau.ac.il}
}
\maketitle

\begin{abstract}
   The single image super-resolution task is one of the most examined inverse problems in the past decade. 
   In the recent years, Deep Neural Networks (DNNs) have shown superior performance over alternative  methods when the acquisition process uses a fixed known downscaling kernel---typically a bicubic kernel.
   However, several recent works have shown that in practical scenarios, where the test data mismatch the training data (e.g. when the downscaling kernel is not the bicubic kernel or is not available at training), the leading DNN methods suffer from a huge performance drop.
   Inspired by the literature on generalized sampling, in this work we propose a method for improving the performance of DNNs that have been trained with a fixed kernel on observations acquired by other kernels.
   For a known 
   kernel, we design a closed-form correction filter that modifies the low-resolution image to match one which is obtained by another kernel (e.g. bicubic), and thus improves the results of existing pre-trained DNNs. For an unknown kernel, we extend this idea and propose an algorithm for blind estimation of the required correction filter.  
   We show that our approach outperforms other super-resolution methods, which are designed for 
   general 
   downscaling kernels.

\end{abstract}

\blfootnote{Code is available at https://github.com/shadyabh/Correction-Filter}

\begin{figure*}[t]
\captionsetup[subfigure]{labelformat=empty}
    \centering
    
    \begin{subfigure}[b]{0.27\linewidth}
        \centering
        \includegraphics[trim={100pt 20pt 100pt 120pt}, clip, width=130pt]{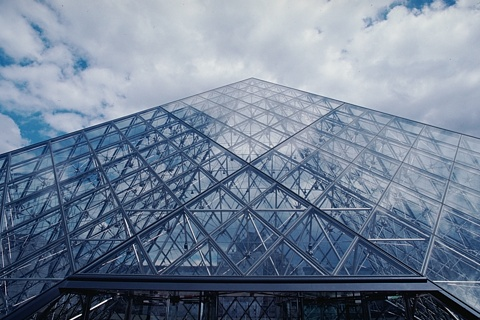}
        \caption{Original image (cropped)}
    \end{subfigure}
     \begin{subfigure}[b]{0.27\linewidth}
        \centering
        \includegraphics[trim={100pt 20pt 100pt 120pt}, clip, width=130pt]{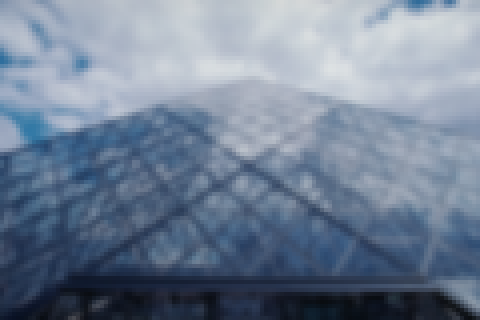}
        \caption{LR image}
    \end{subfigure}\\
    \begin{subfigure}[b]{0.27\linewidth}
        \centering
        \includegraphics[trim={100pt 20pt 100pt 120pt}, clip, width=130pt]{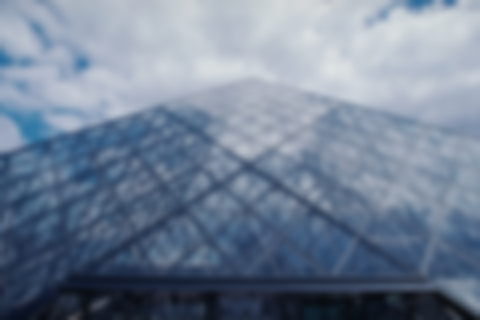}
        \caption{Bicubic upsampling}
    \end{subfigure}
    \begin{subfigure}[b]{0.27\linewidth}
        \centering
        \includegraphics[trim={100pt 20pt 100pt 120pt}, clip, width=130pt]{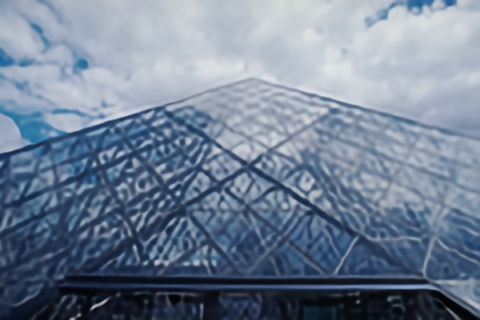}
        \caption{SRMD \cite{SRMD}}
    \end{subfigure}
    \begin{subfigure}[b]{0.27\linewidth}
        \centering
        \includegraphics[trim={100pt 20pt 100pt 120pt}, clip, width=130pt]{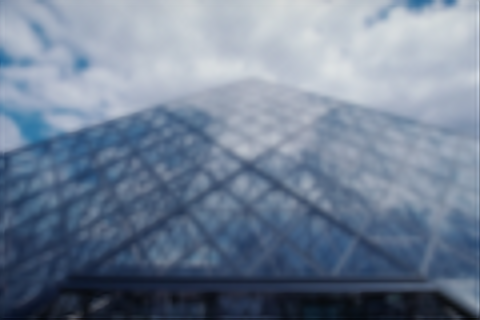}
        \caption{ZSSR \cite{ZSSR}}
    \end{subfigure}\\
    \begin{subfigure}[b]{0.27\linewidth}
        \centering
        \includegraphics[trim={100pt 20pt 100pt 120pt}, clip, width=130pt]{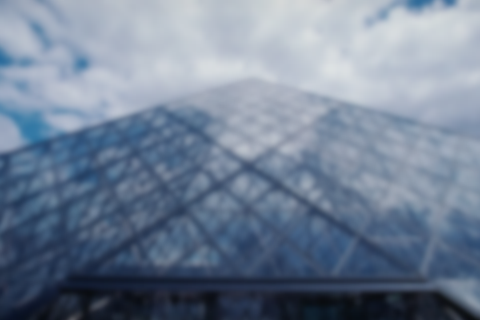}
        \caption{ProSR \cite{proSR}}
    \end{subfigure}
    \begin{subfigure}[b]{0.27\linewidth}
        \centering
        \includegraphics[trim={100pt 20pt 100pt 120pt}, clip, width=130pt]{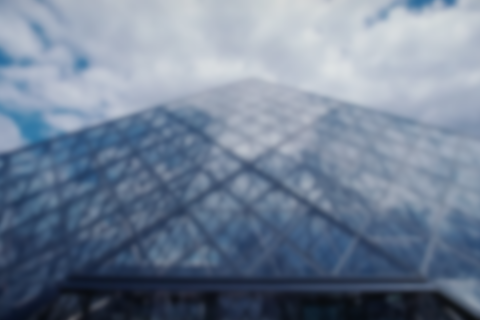}
        \caption{RCAN \cite{RCAN}}
    \end{subfigure}
    \begin{subfigure}[b]{0.27\linewidth}
        \centering
        \includegraphics[trim={100pt 20pt 100pt 120pt}, clip, width=130pt]{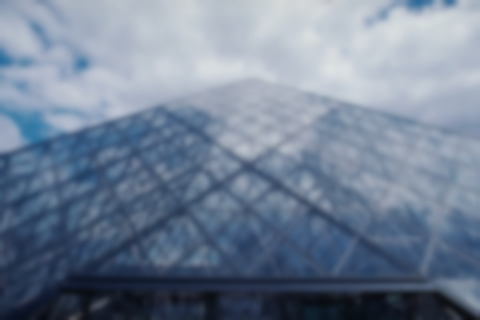}
        \caption{DBPN \cite{DBPN}}
    \end{subfigure}
    \begin{subfigure}[b]{0.27\linewidth}
        \centering
        \includegraphics[trim={100pt 20pt 100pt 120pt}, clip, width=130pt]{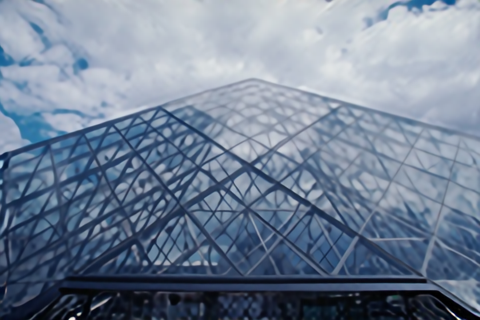}
        \caption{ProSR with our correction}
    \end{subfigure}
    \begin{subfigure}[b]{0.27\linewidth}
        \centering
        \includegraphics[trim={100pt 20pt 100pt 120pt}, clip, width=130pt]{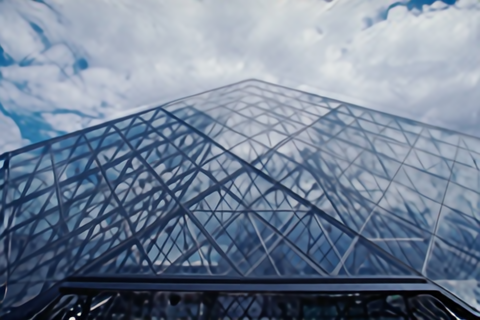}
        \caption{RCAN with our correction}
    \end{subfigure}
    \begin{subfigure}[b]{0.27\linewidth}
        \centering
        \includegraphics[trim={100pt 20pt 100pt 120pt}, clip, width=130pt]{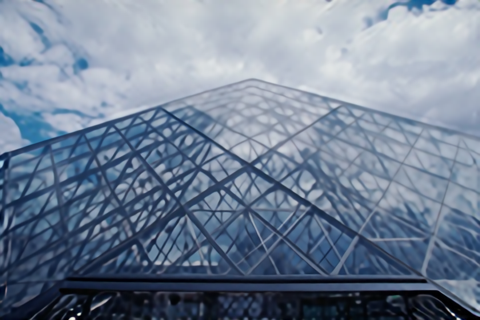}
        \caption{DBPN with our correction}
    \end{subfigure}
    \caption{Non-blind super-resolution of image 223061 from BSD100, for scale factor 4 and Gaussian downscaling kernel with std $4.5/\sqrt{2}$. 
    Our correction filter significantly improves the performance of DNNs trained on another SR kernel.}
    \label{fig:comparison_gauss_4.5}
\end{figure*}

\section{Introduction}
\label{sec:Intro}

The task of Single Image Super-Resolution (SISR) is one of the most examined inverse problems in the past decade \cite{freeman2002example, glasner2009super, yang2010image, dong2012nonlocally}. 
In this problem, the goal is to reconstruct a latent high-resolution (HR) image from its low-resolution (LR) version obtained by an acquisition process that includes low-pass filtering and sub-sampling.
In the recent years, along with the developments in deep learning, many SISR methods that are based on Deep Neural Networks (DNNs) have been proposed \cite{dong2014learning, kim2016accurate, lim2017enhanced, ledig2017photo, RCAN, proSR, DBPN}.

Typically, the performance of SISR approaches is evaluated on test sets with a fixed known acquisition process, e.g. a bicubic downscaling kernel.
This evaluation methodology allows to prepare large training data, which are based on ground truth HR images and their LR counterparts synthetically obtained through the known observation model.
DNNs, which have been exhaustively trained on such training data, clearly outperform other alternative algorithms, e.g. methods that are based on hand-crafted prior models such as sparsity or non-local similarity \cite{dong2012nonlocally,glasner2009super,yang2010image}.

Recently, several works have shown that in practical scenarios where the test data mismatch the training data, the leading DNN methods suffer from a huge performance drop \cite{zhang2017learning, ZSSR, tirer2019super}. 
Such scenarios include a downscaling kernel which is not the bicubic kernel and is not available at the training phase. A primary example is an unknown kernel that needs to be estimated from the LR image at test time.

Several recent SISR approaches have proposed different strategies for enjoying the advantages of deep learning while mitigating the restriction of DNNs to the fixed kernel assumption made in the training phase.
These strategies include: modifying the training phase such that it covers a predefined set of downscaling kernels \cite{SRMD,IKC}; using DNNs to capture only a natural-image prior which is decoupled from the SISR task \cite{zhang2017learning,bora2017compressed}; or completely avoid any offline training and instead train a CNN super-resolver from scratch at test time \cite{ulyanov2018deep,ZSSR}. 

\textbf{Contribution.} 
In this work we take a different strategy, inspired by the generalized sampling literature \cite{eldar2015sampling,unser2000sampling}, for handling LR images obtained by arbitrary downscaling kernels. Instead of completely ignoring the state-of-the-art DNNs that have been trained for the bicubic model, as done by other prior works, we propose a method that transforms the LR image to match one which is obtained by the bicubic kernel.
The modified LR can then be inserted into existing leading super-resolvers, such as DBPN \cite{DBPN}, RCAN \cite{RCAN}, and proSR \cite{proSR}, thus, improving their performance significantly on kernels they have not been trained on.
The proposed transformation is performed using a correction filter, which has a closed-form expression when the true (non-bicubic) kernel is given.

In the "blind" setting, where the kernel is unknown, we extend our approach and propose an algorithm that estimates the required correction. 
The proposed approach outperforms other super-resolution methods 
in various practical scenarios. See example results in Figure~\ref{fig:comparison_gauss_4.5} 
for the non-blind setting.

\begin{figure*}[t]
    \centering\includegraphics[width=450pt]{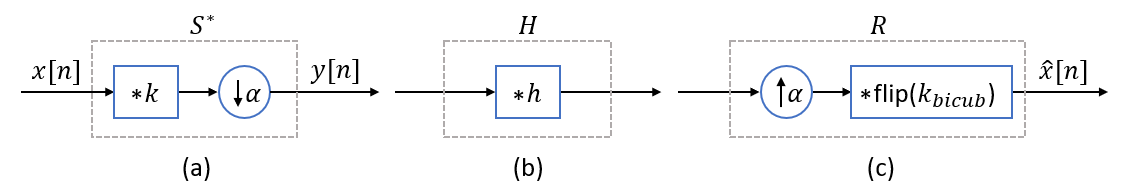}
  \caption{Downsampling, correction, and upsampling operators for single image super-resolution: (a) Downsampling operator composed of convolution with a kernel $\mathbf{k}$ and sub-sampling by factor of $\alpha$; (b) Correction operator composed of convolution with a correction filter $\mathbf{h}$; (c) Upsampling operator composed of up-sampling by factor of $\alpha$ and convolution with a (flipped) kernel $\mathbf{k}_{bicub}$. In our approach, $\mathcal{H}$ is computed for $\mathcal{S}^*$ and $\mathcal{R}$, but then we replace $\mathcal{R}$ with a pre-trained DNN super-resolver.}
\label{fig:block_diag}
\end{figure*}

\section{Related Work}
\label{sec:related}

In the past five years many works have employed DNNs for the SISR task, showing a great advance in performance with respect to the reconstruction error \cite{dong2014learning, DBPN, kim2016accurate, lim2017enhanced, proSR, RCAN} and the perceptual quality \cite{Blau2018Perception,ledig2017photo, sajjadi2017enhancenet, wang2018esrgan}. 
However, one main disadvantage of DNN-based SISR methods is their sensitivity to the LR image formation model. A network performance tends to drop significantly if it has been trained for one acquisition model and then been tested on another \cite{ZSSR, tirer2019super,zhang2017learning}.

Recently, different SISR strategies has been proposed with the goal of enjoying the advantages of deep learning while mitigating the restriction of DNNs to the fixed kernel assumption made in the training phase.
One approach trains a CNN super-resolver that gets as inputs both the LR image and the degradation model, and assumes that the downscaling kernels belong to a certain set of Gaussian filters \cite{IKC,SRMD}.
Another approach builds on the structural prior of CNNs, which promotes signals with spatially recurring patterns (e.g.~natural images) and thus allows to 
train a super-resolver CNN from scratch at test time \cite{ZSSR,ulyanov2018deep}. 
Another line of work recovers the latent HR image by minimizing a cost function 
that is composed of a fidelity term (e.g.~least squares or back-projection \cite{tirer2019back}) and a prior term, where only the latter 
is handled by a pre-trained CNN denoiser or GAN  \cite{zhang2017learning,tirer2018image,bora2017compressed,DPSR}. 
Recently, the two last approaches have been incorporated by applying image-adaptation to denoisers \cite{tirer2019super} and GANs \cite{shady2019image}.
In all these methods the downscaling kernel is given as an input. In a blind setting (where the kernel is unknown) it is still possible to apply these methods after an initial kernel estimation phase.

Our approach is inspired by the literature on generalized sampling \cite{Butzer1983Survey,unser2000sampling,eldar2015sampling}, which generalizes the classical Whittaker–Nyquist–Kotelnikov–Shannon sampling theorem \cite{whittaker1915xviii, nyquist1928certain, kotelnikov1933transmission,shannon1949communication}, which considers signals that are band-limited in the frequency domain and sinc interpolations.
The generalized theory provides a framework and conditions under which a signal that is sampled by a certain basis can be reconstructed by a different basis. In this framework, the sampled signal is reconstructed using a linear operator that can be decoupled into two phases, the first applies a digital correction filter and the second includes processing with a reconstruction kernel. 
The role of the correction filter is to transform the sampling coefficients, associated with the sampling kernel, to coefficients which fit the reconstruction kernel.

Several 
works have used the 
correction filter approach for image processing \cite{eldar2009beyond,glasbey2001optimal,ramani2006non}. 
These works typically propose linear interpolation methods, i.e.~the correction filter is followed by a linear reconstruction operation, and do not use a strong natural-image prior. As a result, the recovery of fine details is lacking.   


In this work, we plan to use (very) non-linear reconstruction methods, namely---DNNs, 
whose training 
is difficult, computationally expensive, storage demanding, and cannot be done when the observation model is not known in advance. To tackle these difficulties, we revive the correction filter approach and show how it can be used with deep super-resolvers which have been already trained.

The required correction filter depends on the kernel which is used for sampling. Therefore, in the blind setting, it needs to be estimated from the LR image. 
To this end, we propose an iterative optimization algorithm for estimating the correction filter. 
In general, only a few works have considered the blind SISR setting and developed kernel estimation methods \cite{sroubek2007unified,michaeli2013nonparametric,IKC,kernelGAN}. 

Finally, we would like to highlight major differences between this paper and the work in \cite{IKC}, whose "kernel correction" approach may be misunderstood as our "correction filter".
In \cite{IKC}, three different DNNs (super-resolver, kernel estimator, and kernel corrector) are offline trained under the assumption that the downscaling kernel belongs to a certain family of Gaussian filters (similarly to \cite{SRMD}), and the CNN super-resolver gets the estimated kernel as an input.
So, the first major difference is that contrary to our approach, no pre-trained existing DNN methods (other than SRMD \cite{SRMD}) can be used in \cite{IKC}.
Secondly, their approach is restricted by the offline training assumptions to very certain type of downscaling kernels, contrary to our approach. 
Thirdly, the concepts of these works are very different: The (iterative) correction in \cite{IKC} modifies the estimated downscaling kernel, while our correction filter modifies the LR image.


\section{The Proposed Method}
\label{Sec:method}

The single image super resolution (SISR) acquisition model, can be formulated as
\begin{equation}
    \label{Eq_downsample_model}
    \mathbf{y} = (\mathbf{x} \ast \mathbf{k})\downarrow _\alpha, 
\end{equation}
where $\mathbf{x} \in \mathbb{R}^n$ represents the latent HR image, $\mathbf{y} \in \mathbb{R}^m$ represents the observed LR image, $\mathbf{k} \in \mathbb{R}^d$ $(d \ll n)$ is the (anti-aliasing) blur kernel, 
$\ast$ denotes the linear convolution operator, and $\downarrow _\alpha$ denotes sub-sampling operator with stride of $\alpha$.
Under the common fashion of dropping the edges of $\mathbf{x} \ast \mathbf{k}$, such that it is in $\mathbb{R}^n$, we have that $m=\ceil*{n/\alpha}$.

Note that Equation \eqref{Eq_downsample_model} can be written in a more elegant way as
\begin{equation}
    \label{Eq_downsample_model2}
    \mathbf{y} = \mathcal{S}^* \mathbf{x}, 
\end{equation}
where $\mathcal{S}^*:\mathbb{R}^n \rightarrow \mathbb{R}^{m}$ is a linear operator that encapsulates the entire downsampling operation, i.e. $\mathcal{S}^*$ is a composition of blurring followed by sub-sampling.
The {\em downsampling operator} $\mathcal{S}^*$ is presented in Figure~\ref{fig:block_diag}(a).

Most SISR deep learning methods, e.g. \cite{dong2014learning, kim2016accurate, lim2017enhanced, ledig2017photo, RCAN, proSR, DBPN}, assume that the observations are obtained using the bicubic kernel. Let us denote by $\mathcal{R}^*$ the associated downsampling operator (essentially, $\mathcal{R}^*$ coincides with the previously defined $\mathcal{S}^*$ if $\mathbf{k}$ is the bicubic kernel $\mathbf{k}_{bicub}$). 
The core idea of our approach is to modify the observations $\mathbf{y} = \mathcal{S}^* \mathbf{x}$,  obtained for an arbitrary downscaling kernel $\mathbf{k}$, such that they mimic the (unknown) "ideal observations" $\mathbf{y}_{bicub} = \mathcal{R}^* \mathbf{x}$, which can be fed into pre-trained DNN models.

In what follows, we present a method to (approximately) achieve this goal using the correction filter tool, adopted from the generalized sampling literature. 
First, we consider the non-blind setting, where the downscaling kernel is known, and thus $\mathcal{S}^*$ is known. In this case, we obtain a closed-form expression for the required correction filter, which depends on $\mathbf{k}$ (and on $\mathbf{k}_{bicub}$).
Later, we extend the approach to the blind setting, where $\mathbf{k}$ is unknown. In this case, we propose a technique for estimating the correction filter from the LR image $\mathbf{y}$.

\subsection{The non-blind setting}
\label{Sec:method_nonblind}

In the non-blind setting, both the downscaling kernel $\mathbf{k}$ and the target kernel $\mathbf{k}_{bicub}$ are known. Therefore, the downsampling operators $\mathcal{S}^*$ and $\mathcal{R}^*$ are known as well.
Using common notations from generalized sampling literature \cite{eldar2015sampling}, let us denote by $\mathcal{S}$ and $\mathcal{R}$ the adjoint operators of $\mathcal{S}^*$ and $\mathcal{R}^*$, respectively.
The operator $\mathcal{R}:\mathbb{R}^m \rightarrow \mathbb{R}^n$ 
is an {\em upsampling operator} 
that restores a signal in $\mathbb{R}^n$ from $m$ samples, associated with the downsampling operator $\mathcal{R}^*$. 
In the context of our work, when $\mathcal{R}$ is applied on a vector it pads it with $n-m$ zeros ($\alpha-1$ zeros between each two entries) and convolves it with a flipped version of $\mathbf{k}_{bicub}$.
The upsampling operator $\mathcal{R}$ is presented in Figure~\ref{fig:block_diag}(c).
A similar definition holds for $\mathcal{S}:\mathbb{R}^m \rightarrow \mathbb{R}^n$ with the kernel $\mathbf{k}$.


The key goal of generalized sampling theory is to identify signal models and sampling systems 
that allow for perfect recovery. Therefore, to proceed, let us make the following assumption.

\begin{assump}
\label{assump:xRRx}
The signal $\mathbf{x}$ can be perfectly recovered from its samples $\mathcal{R}^*\mathbf{x}$ by the operator $\mathcal{R} (\mathcal{R}^*\mathcal{R})^{-1}$, i.e. 
\begin{equation}
    \mathbf{x} = \mathcal{R} (\mathcal{R}^*\mathcal{R})^{-1} \mathcal{R}^* \mathbf{x}.
\end{equation}
\end{assump}

Assumption~\ref{assump:xRRx} essentially states that the latent image $\mathbf{x}$ resides in the linear subspace spanned by the bicubic kernel. Therefore, it can be perfectly recovered from the observations $\mathbf{y}_{bicub} = \mathcal{R}^* \mathbf{x}$ by applying the pseudoinverse of $\mathcal{R}^*$ on $\mathbf{y}_{bicub}$, i.e.~by the estimator $\hat{\mathbf{x}}=\mathcal{R}(\mathcal{R}^*\mathcal{R})^{-1}\mathbf{y}_{bicub}$. 
Even though Assumption~\ref{assump:xRRx} does not hold for natural images, 
it is motivated by the fact that 
there are 
many DNN methods that can handle observations of the form $\mathcal{R}^* \mathbf{x}$ quite well.

However, since we are given observations that are obtained by a different downscaling kernel, $\mathbf{y}=\mathcal{S}^* \mathbf{x}$, let us propose a different estimator $\hat{\mathbf{x}}=\mathcal{R}\mathcal{H}\mathbf{y}$, where $\mathcal{H}:\mathbb{R}^m \rightarrow \mathbb{R}^m$ is a {\em correction operator}.
This recovery procedure is presented in Figures~\ref{fig:block_diag}(b)+\ref{fig:block_diag}(c).
The following theorem presents a condition and a formula for $\mathcal{H}$ under which perfect recovery is possible under Assumption~\ref{assump:xRRx}.

\begin{thm}
\label{thm:correctFilt}
Let $\mathbf{y}=\mathcal{S}^* \mathbf{x}$, $\hat{\mathbf{x}}=\mathcal{R}\mathcal{H}\mathbf{y}$, and assume that Assumption~\ref{assump:xRRx} holds. Then, if
\begin{equation}
\label{Eq_emptyNull}
    \mathrm{null}(\mathcal{S}^*) \cap \mathrm{range}(\mathcal{R}) = \{0\}, 
\end{equation}
we have that $\hat{\mathbf{x}}=\mathbf{x}$ for
\begin{equation}
\label{Eq_corrFilt}
    \mathcal{H} = \left ( \mathcal{S}^* \mathcal{R} \right )^{-1} : \mathbb{R}^m \rightarrow \mathbb{R}^m.
\end{equation}
\end{thm}

\begin{proof}

Note that 
\begin{align}
\hat{\mathbf{x}}&=\mathcal{R}\mathcal{H}\mathbf{y} \nonumber \\
&=\mathcal{R}\mathcal{H}\mathcal{S}^* \mathbf{x} \nonumber \\
&=\mathcal{R}\mathcal{H}\mathcal{S}^* \mathcal{R}(\mathcal{R}^*\mathcal{R})^{-1} \mathcal{R}^* \mathbf{x},
\end{align}
where the last equality follows from Assumption~\ref{assump:xRRx}. Next, \eqref{Eq_emptyNull} implies that the operator $(\mathcal{S}^* \mathcal{R})$ is invertible. Thus, setting $\mathcal{H}$ according to \eqref{Eq_corrFilt} is possible, and we get
\begin{align}
\hat{\mathbf{x}}&=\mathcal{R} \left ( \mathcal{S}^* \mathcal{R} \right )^{-1} \mathcal{S}^* \mathcal{R} (\mathcal{R}^*\mathcal{R})^{-1} \mathcal{R}^* \mathbf{x} \nonumber \\
&=\mathcal{R} (\mathcal{R}^*\mathcal{R})^{-1} \mathcal{R}^* \mathbf{x} = \mathbf{x},
\end{align}
where the last equality follows from Assumption~\ref{assump:xRRx}.

\end{proof}

Theorem~\ref{thm:correctFilt} is presented in operator notations to simplify the derivation.
In the context of SISR (i.e.~with the previous definitions of $\mathcal{S}^*$ and $\mathcal{R}$), 
the operator $\mathcal{H} = \left ( \mathcal{S}^* \mathcal{R} \right )^{-1}$ 
can be applied simply as a convolution with a correction filter $\mathbf{h}_0$, 
given by
\begin{equation}
\label{Eq_corrFilt3}
    \mathbf{h}_0 = \mathrm{IDFT} \left \{ \frac{1}{ \mathrm{DFT}\left \{ (\mathbf{k}\ast \mathrm{flip}(\mathbf{k}_{bicub})) \downarrow_{\alpha} \right\} } \right \},
\end{equation}
where $\mathrm{DFT}(\cdot)$ and $\mathrm{IDFT}(\cdot)$ denote the Discrete Fourier Transform and its inverse respectively.\footnote{Using DFT allows fast  implementation of cyclic convolutions. When it is used for linear convolutions, edge artifacts need to be ignored.} 
%


\begin{table*}
\footnotesize
\renewcommand{\arraystretch}{1.3}
\caption{Non-blind super-resolution comparison on Set14. Each cell displays PSNR [dB] (left) and SSIM (right).}
\label{table:super_res}
\centering
    \begin{tabular}{ | l || l | l | l | l |}
    \hline
         & Scale & Gaussian std = $1.5/\sqrt{2}$ & Gaussian std = $2.5/\sqrt{2}$ & Box of width = 4 \\ \hline 
    ZSSR & 2 & 28.107 / 0.829 & 27.954 / 0.806 & 28.506 / 0.802 \\ \hline
    SRMD & 2 & 32.493 / 0.878 & 29.923 / 0.812 & 25.944 / 0.757 \\ \hline 
    DBPN &   2	 & 30.078 / 0.850& 26.366 / 0.734& 28.444 / 0.803\\ 
    DBPN + our correction & 2 & 34.023 / \textbf{0.904} & \textbf{33.288} / \textbf{0.895} & 29.364 / 0.822 \\ \hline 
    ProSR & 2 & 30.073 / 0.849 & 26.371 / 0.734 & 28.459 / 0.803 \\ 
    ProSR + our correction & 2 & 33.954 / 0.903 & 33.273 / \textbf{0.895} & \textbf{29.514} / \textbf{0.825} \\ \hline 
    RCAN & 2 & 30.118 / 0.851 & 26.389 / 0.736 & 28.469 / 0.804 \\ 
    RCAN + our correction & 2 & \textbf{34.043} / \textbf{0.904} & 33.251 / \textbf{0.895} & 29.306 / 0.820 \\ \hline 

    \end{tabular}
    \begin{tabular}{ | l || l | l | l | l |}
    \hline 
         & Scale & Gaussian std = $3.5/\sqrt{2}$ & Gaussian std = $4.5/\sqrt{2}$ & Box of width = 8 \\ \hline 
    ZSSR & 4 & 25.642 / 0.701 & 25.361 / 0.683 & 24.549 / 0.653 \\ \hline
    SRMD & 4 & 26.877 / 0.718 & 25.350 / 0.674 & 19.704 / 0.525 \\ \hline  
    DBPN & 4	 & 25.067 / 0.685 & 23.890 / 0.645 & 24.636 / 0.667 \\ 
    DBPN + our correction & 4 & \textbf{28.680} / \textbf{0.775} & \textbf{28.267} / \textbf{0.766} & 25.157 / 0.679 \\ \hline 
    ProSR & 4 & 25.033 / 0.683 & 23.882 / 0.645 & 24.685 / 0.667 \\ 
    ProSR + our correction & 4 & 28.609 / 0.772 & 28.220 / 0.764 & \textbf{25.419} / \textbf{0.683} \\ \hline 
    RCAN & 4 & 25.077 / 0.685 & 23.904 / 0.646 & 24.694 / 0.668 \\ 
    RCAN + our correction & 4 & 28.534 / 0.771 & 28.110 / 0.762 & 25.301 / 0.679\\ \hline 

    \end{tabular}
\end{table*}

In practice, instead of using the weak estimator $\hat{\mathbf{x}}=\mathcal{R}\mathcal{H}\mathbf{y}$ that does not use any natural-image prior, we propose to recover the HR image by
\begin{equation}
\label{Eq_our_SR}
    \hat{\mathbf{x}}=f(\mathbf{h} \ast \mathbf{y}),
\end{equation}
where $f(\cdot)$ is a DNN super-resolver that has been trained under the assumption of bicubic kernel 
and $\mathbf{h}$ is a modified correction filter, given by
\begin{align}
\label{Eq_corrFilt2p5}
    \mathbf{h} &= \mathrm{IDFT} \left \{ \frac{\mathrm{DFT}\left \{ (\mathbf{k}_{bicub}\ast \mathrm{flip}(\mathbf{k}_{bicub})) \downarrow_{\alpha} \right\}}{ \mathrm{DFT}\left \{ (\mathbf{k}\ast \mathrm{flip}(\mathbf{k}_{bicub})) \downarrow_{\alpha} \right\}} \right \}  \nonumber \\
    & \triangleq \mathrm{IDFT} \left \{ \frac{ F_{numer} }{ F_{denom} } \right \}.
\end{align}
Let us explain the idea behind the estimator in \eqref{Eq_our_SR}. Since the inverse mapping $f(\cdot)$ assumes bicubic downscaling ($\mathcal{R}^* \mathbf{x}$), it can be interpreted as incorporating $\mathcal{R} (\mathcal{R}^*\mathcal{R})^{-1}$ with a (learned) prior. Therefore, unlike $\mathbf{h}_0$ in \eqref{Eq_corrFilt3} which is followed by $\mathcal{R}$, here the correction filter should also compensate for the operation $(\mathcal{R}^*\mathcal{R})^{-1}$ which is implicitly done in $f(\cdot)$. This explains the term in the numerator of $\mathbf{h}$ (compared to 1 in the numerator of $\mathbf{h}_0$).
To ensure numerical stability, we slightly modify \eqref{Eq_corrFilt2p5}, and compute $\mathbf{h}$ using
\begin{equation}
\label{Eq_corrFilt2}
    \mathbf{h} = \mathrm{IDFT} \left \{ F_{numer} \cdot \frac{ F_{denom}^* }{ \left| F_{denom} \right |^2 + \epsilon } \right \},
\end{equation}
where $\epsilon$ is a small regularization parameter. 
Regarding the choice of $f(\cdot)$, 
in our experiments 
we use DBPN \cite{DBPN}, RCAN \cite{RCAN}, and proSR \cite{proSR}, but in general any other method with state-of-the-art performance (for bicubic kernel) is expected to give good results.

Note that the theoretical motivation for our strategy requires that the condition in \eqref{Eq_emptyNull} holds.
This condition can be inspected by compaing the bandwidth of the kernels $\mathbf{k}$ and $\mathbf{k}_{bicub}$ in the frequency domain. As $\mathbf{k}$ is commonly a low-pass filter (and so is $\mathbf{k}_{bicub}$), 
the condition requires that the passband of $\mathbf{k}_{bicub}$ is contained in the passband of $\mathbf{k}$. 
Yet, as shown in the experiments section, 
our approach yields a significant improvement even when the passband of $\mathbf{k}$ is moderately smaller than the passband of $\mathbf{k}_{bicub}$.
Furthermore, we observe that even for very blurry LR images performance improvement can be obtained by increasing the regularization parameter in \eqref{Eq_corrFilt2}.
We refer the reader to the supplementary material for more details.

\begin{table*}
\footnotesize
\renewcommand{\arraystretch}{1.3}
\caption{Non-blind super-resolution comparison on BSD100. Each cell displays PSNR [dB] (left) and SSIM (right).}
\label{table:super_res_BSD100}
\centering
    
    \begin{tabular}{ | l || l | l | l || l | l | l | l |}
    \hline
         & Scale & Gaussian std = $1.5/\sqrt{2}$ & Gaussian std = $2.5/\sqrt{2}$ & Scale & Gaussian std = $3.5/\sqrt{2}$ & Gaussian std = $4.5/\sqrt{2}$ \\ \hline
        ZSSR & 2 & 29.339 / 0.822 & 26.415 / 0.715 & 4 & 25.115 / 0.651 & 24.348 / 0.625 \\ \hline
        SRMD & 2 & 26.591 / 0.803 & 29.294 / 0.838 & 4 & 25.735 / 0.704 & 26.432 / 0.707 \\ \hline        
        DBPN &  2 & 29.512 / 0.827 & 26.371 / 0.711 & 4 & 25.268 / 0.662 & 24.357 / 0.628 \\ 
        DBPN + correction & 2 & 32.300 / 0.884 & 31.875 / \textbf{0.878} & 4 & \textbf{27.690} / \textbf{0.740} & \textbf{27.474} / \textbf{0.733} \\ \hline
        ProSR & 2 & 29.513 / 0.827 & 26.381 / 0.711 & 4 &25.237 / 0.661 & 24.353 / 0.628 \\ 
        ProSR + correction & 2 & 32.276 / 0.884 & 31.899 / \textbf{0.878} & 4 & 27.645 / 0.738  & 27.455/ \textbf{0.733}  \\ \hline
        RCAN & 2 & 29.558 / 0.829 & 26.397 / 0.713 & 4 & 25.281 / 0.663  & 24.373 / 0.629  \\
        RCAN + correction &  2 & \textbf{32.368} / \textbf{0.886} & \textbf{31.876} / \textbf{0.878} & 4 & 27.626 / 0.739 & 27.399/ 0.732 \\ \hline
    \end{tabular}

\end{table*}







\subsection{The blind setting}
\label{Sec:method_blind}

In the blind setting, the downscaling kernel $\mathbf{k}$ is unknown. Therefore, we cannot compute the correction filter $\mathbf{h}$ using \eqref{Eq_corrFilt2}, and extending our approach to this setting  
requires to estimate $\mathbf{k}$ and $\mathbf{h}$ from the LR image $\mathbf{y}$.

To this end, we propose 
to estimate $\mathbf{k}$ as the minimizer of the following objective function 
\begin{equation}
\label{Eq_estk_cost}
    \xi(\mathbf{k}) = \|\mathbf{y} - \mathcal{S}^* f(\mathcal{H}\mathbf{y}) \|_{\mathrm{Hub}} + \| \mathbf{m}_{\mathrm{cen}} \cdot \mathbf{k} \|_1 + \|\mathbf{k}\|_1,
\end{equation}
where $\|\cdot\|_{\mathrm{Hub}}$ is Huber loss \cite{huber1964robust}, the operator $\mathcal{H}$ is filtering with $\mathbf{h}$ given in \eqref{Eq_corrFilt2}, $\mathcal{S^*}$ is the downsampling operator, $f(\cdot)$ is the given SR network, and $\mathbf{m}_{\mathrm{cen}}$ is given by
$$\mathbf{m}_{\mathrm{cen}}(x, y) = 1 - \mathrm{e}^{-\frac{(x^2 + y^2)}{32\alpha^2}},$$ 
where $\alpha$ is the scale factor. Note that the two operators $\mathcal{H}$ and $\mathcal{S^*}$ depend on the kernel $\mathbf{k}$. The last two terms in \eqref{Eq_estk_cost} are regularizers: 
the last term promotes sparsity of $\mathbf{k}$ and the penultimate term centralizes its density. 

\begin{algorithm}
\caption{Correction filter estimation}
\vspace{2mm}
\kwInput{$\mathbf{y}$, $\mathbf{k}_{bicub}$, $\alpha$, $f(\cdot)$.}
\kwOutput{$\hat{\mathbf{h}}$ an estimate for $\mathbf{h}$.}
\kwParameters{$\mathbf{k}^{(0)}=\mathbf{k}_{bicub}$, $i=0$, $\epsilon=10^{-14}$, $\gamma=10^{-4}$, $N_{iter} = 250$} 
\While{$i<N_{iter}$}{
    $i = i+1$\;
    Compute $\mathbf{h}^{(i)}$ using \eqref{Eq_corrFilt2} (for $\mathbf{k}^{(i-1)}$, $\alpha$ and $\epsilon$)\;
    $\mathbf{x}_{h}^{(i)} = f(\mathbf{h}^{(i)}\ast \mathbf{y})$\;
    $\hat{\mathbf{y}}^{(i)} = (\mathbf{x}_{h}^{(i)} \ast \mathbf{k}^{(i-1)})\downarrow _\alpha$\;
    $\xi(\mathbf{k}^{(i-1)}) = \| \mathbf{y} - \hat{\mathbf{y}}^{(i)} \|_{\mathrm{Hub}} + \| \mathbf{m}_{\mathrm{cen}} \cdot \mathbf{k}^{(i-1)} \|_1 + \| \mathbf{k}^{(i-1)} \|_1$\;
    $\mathbf{k}^{(i)}=$ Adam update (for $\xi(\mathbf{k}^{(i-1)})$ with LR $\gamma$)\;
}
$\hat{\mathbf{h}} = \mathbf{h}^{(i)}$\;
\label{Alg:h_est}
\end{algorithm}

Inspired by \cite{arora2018convergence, kernelGAN}, we choose to parameterize the latent $\mathbf{k}$ by a {\em linear} CNN composed of 4 layers, i.e. $\mathbf{k} = \mathbf{k}_0\ast \mathbf{k}_1\ast \mathbf{k}_2 \ast \mathbf{k}_3$, where $\left\{\mathbf{k}_n\right\}_{n=0}^{2}$ are of size $33 \times 33$ and $\mathbf{k}_3$ is of size $32 \times 32$. 
The minimization of \eqref{Eq_estk_cost} with respect to $\mathbf{k}$ is performed by 250 iterations of Adam \cite{Adam} with learning rate of $10^{-4}$, initialized with $\mathbf{k}^{(0)} = \mathbf{k}_{bicub}$. 
The proposed procedure is described in Algorithm~\ref{Alg:h_est}.
Note that at each iteration we obtain estimates for both the downscaling kernel $\mathbf{k}$ and correction filter $\mathbf{h}$. 
The final estimator of $\mathbf{h}$ is then used in \eqref{Eq_our_SR} to reconstruct the HR image, similarly to the non-blind setting.



\section{Experiments}
\label{sec:exp}

In this section we examine the performance and improvement due to our correction filter approach in the non-blind and blind settings, 
using three different off-the-shelf DNN super-resolvers that serve as $f(\cdot)$ in \eqref{Eq_our_SR}: DBPN \cite{DBPN}, RCAN \cite{RCAN}, and proSR \cite{proSR}.
We compare our approach to other methods that receive the downscaling kernel $\mathbf{k}$ (or its estimation in the blind setting) as an input: ZSSR \cite{ZSSR} and SRMD \cite{SRMD}.
We also compare our method to DPSR \cite{DPSR}, however, since its results are extremely inferior to the other strategies (e.g. about 10 dB lower PSNR) they are deferred to the supplementary material.
All the experiments are performed with the official code of each method.  
We refer the reader to the supplementary material for more results.

\begin{figure}
\captionsetup[subfigure]{labelformat=empty}
    \centering
    \begin{subfigure}[b]{0.50\linewidth}
        \centering
        \includegraphics[trim={25pt 5pt 25pt 30pt}, clip, width=75pt]{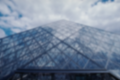}
        \caption{Observed LR (Gaussian kernel)}
    \end{subfigure}\\
     \begin{subfigure}[b]{0.45\linewidth}
        \centering
        \includegraphics[trim={25pt 5pt 25pt 30pt}, clip, width=75pt]{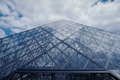}
        \caption{"Ideal" LR (bicubic kernel)}
    \end{subfigure}
    \begin{subfigure}[b]{0.45\linewidth}
        \centering
        \includegraphics[trim={25pt 5pt 25pt 30pt}, clip, width=75pt]{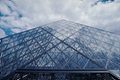}
        \caption{Corrected observed LR}
    \end{subfigure}
    \caption{Comparison between observed, "ideal" and corrected LR images, for SR scale factor of 4 and Gaussian kernel of std $4.5/\sqrt{2}$. The SR result appears in Figure~\ref{fig:comparison_gauss_4.5}}
    \label{fig:correctedVSBicubic}
\end{figure}

\subsection{The non-blind setting}
\label{sec:exp_nonblind}

In this section, we assume that the downscaling kernel $\mathbf{k}$ is known.
Therefore, the correction filter $\mathbf{h}$ can be computed directly using \eqref{Eq_corrFilt2}.
We examine scenarios with scale factors of 2 and 4.
For scale factor of 2, we use Gaussian kernels with standard deviation $\sigma=1.5/\sqrt{2}$ and $\sigma=2.5/\sqrt{2}$, and box kernel of size $4 \times 4$.
For scale factor of 4, we use Gaussian kernels with standard deviation $\sigma=3.5/\sqrt{2}$ and $\sigma=4.5/\sqrt{2}$, and box kernel of size $8 \times 8$.

\begin{table}
\footnotesize
\renewcommand{\arraystretch}{1.3}
\caption{Comparison of the filter-corrected (non-bicubic) LR to the bicubic LR on Set14. Each cell displays PSNR [dB] (left) and SSIM (right).}
\label{table:corrected_compared_to_bicubic}
\centering
    \begin{tabular}{| l | l | l | l |}
    \hline
         Scale & Gauss. std $1.5/\sqrt{2}$ & Gauss. std $2.5/\sqrt{2}$ & Box, width 4 \\ \hline 
       	2 & 51.345 / 0.999 & 45.456 / 0.995 & 33.679 / 0.941 \\ \hline \hline
     Scale & Gauss. std $3.5/\sqrt{2}$ & Gauss. std $4.5/\sqrt{2}$ & Box, width 8 \\ \hline 
     4 & 58.437 / 0.999 & 46.917 / 0.995 & 32.308 / 0.907 \\ \hline
    \end{tabular}
\end{table}

\begin{figure*}
\captionsetup[subfigure]{labelformat=empty}
    \centering
    \begin{subfigure}[b]{0.30\linewidth}
        \centering
        \includegraphics[trim={60pt 120pt 60pt 200pt}, clip, width=107pt]{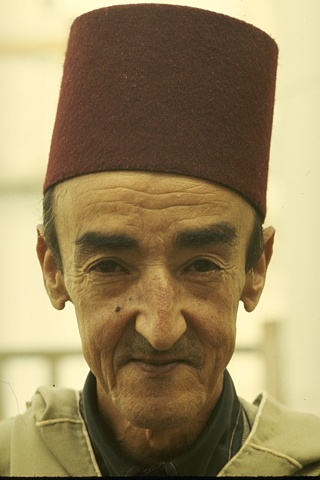}
        \caption{Original image (cropped)}
    \end{subfigure}
    \begin{subfigure}[b]{0.30\linewidth}
        \centering
        \includegraphics[trim={60pt 120pt 60pt 200pt}, clip, width=107pt]{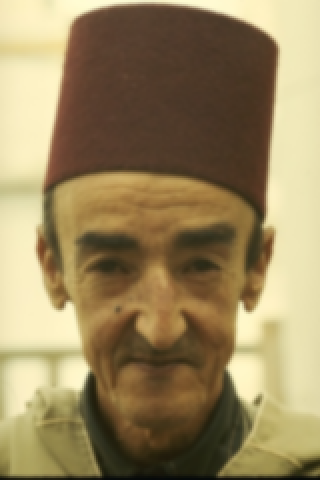}
        \caption{LR image}
    \end{subfigure}\\
    \begin{subfigure}[b]{0.26\linewidth}
        \centering
        \includegraphics[trim={60pt 120pt 60pt 200pt}, clip, width=107pt]{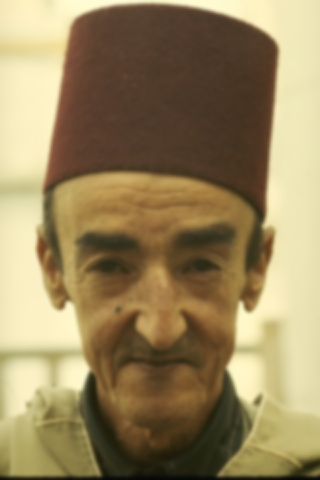}
        \caption{Bicubic upsampling}
    \end{subfigure}
    \begin{subfigure}[b]{0.26\linewidth}
        \centering
        \includegraphics[trim={60pt 120pt 60pt 200pt}, clip, width=107pt]{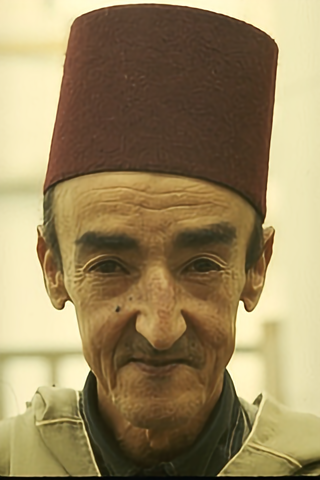}
        \caption{SRMD}
    \end{subfigure}
    \begin{subfigure}[b]{0.26\linewidth}
        \centering
        \includegraphics[trim={60pt 120pt 60pt 200pt}, clip, width=107pt]{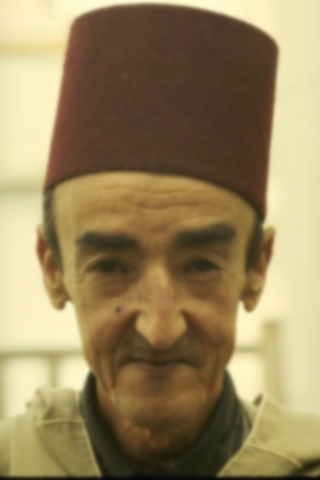}
        \caption{ZSSR}
    \end{subfigure}\\
    \begin{subfigure}[b]{0.26\linewidth}
        \centering
        \includegraphics[trim={60pt 120pt 60pt 200pt}, clip, width=107pt]{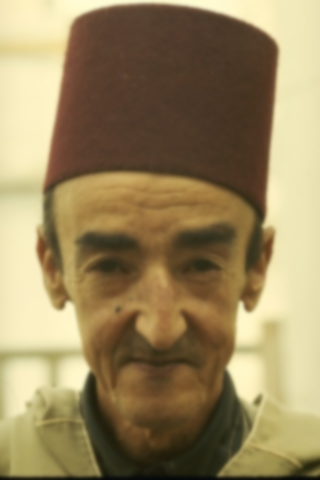}
        \caption{ProSR}
    \end{subfigure}
    \begin{subfigure}[b]{0.26\linewidth}
        \centering
        \includegraphics[trim={60pt 120pt 60pt 200pt}, clip, width=107pt]{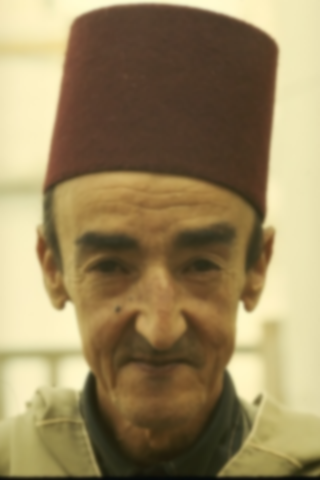}
        \caption{RCAN}
    \end{subfigure}
    \begin{subfigure}[b]{0.26\linewidth}
        \centering
        \includegraphics[trim={60pt 120pt 60pt 200pt}, clip, width=107pt]{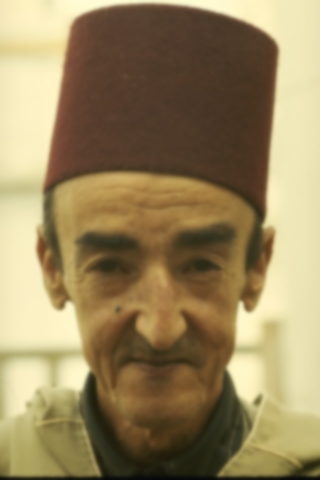}
        \caption{DBPN}
    \end{subfigure} \\ 
    \begin{subfigure}[b]{0.26\linewidth}
        \centering
        \includegraphics[trim={60pt 120pt 60pt 200pt}, clip, width=107pt]{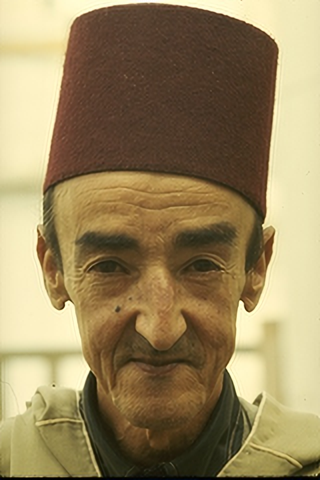}
        \caption{ProSR with our correction}
    \end{subfigure}
    \begin{subfigure}[b]{0.26\linewidth}
        \centering
        \includegraphics[trim={60pt 120pt 60pt 200pt}, clip, width=107pt]{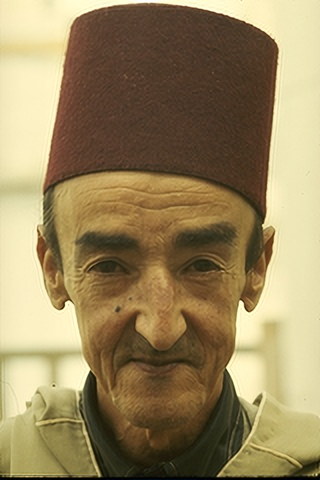}
        \caption{RCAN with our correction}
    \end{subfigure}
    \begin{subfigure}[b]{0.26\linewidth}
        \centering
        \includegraphics[trim={60pt 120pt 60pt 200pt}, clip, width=107pt]{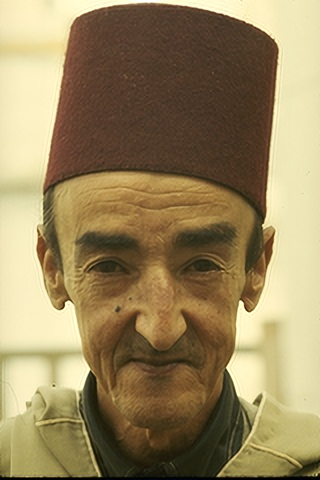}
        \caption{DBPN with our correction}
    \end{subfigure}
    \caption{Non-blind super-resolution of image 189080 from BSD100, for scale factor of 2 and Gaussian downscaling kernel with standard deviation 2.5.} 
    \label{fig:comparison_gauss_2.5}
\end{figure*}

The results are presented in Tables~\ref{table:super_res} and \ref{table:super_res_BSD100} for the test-sets Set14 and BSD100, respectively.
Figures~\ref{fig:comparison_gauss_4.5} and \ref{fig:comparison_gauss_2.5} 
present several visual results.
It can be seen that the proposed filter correction approach significantly improves the results of DBPN, RCAN, and proSR, which have been trained for the (incorrect) bicubic kernel.
Moreoever, note that the filter-corrected applications of DBPN, RCAN, and proSR, also outperform SRMD and ZSSR, while the plain applications of DBPN, RCAN, and proSR are inferior to SRMD.

As explained in Section~\ref{Sec:method}, the proposed approach is based on mimicking the (unknown) "ideal" LR image $\mathbf{y}_{bicub} = \mathcal{R}^* \mathbf{x}$ by the corrected LR image $\mathcal{H}\mathbf{y}$. The high PSNR and SSIM results between such pairs of images, which are presented in Table~\ref{table:corrected_compared_to_bicubic} for Set14, verify that this indeed happens. 
Figure~\ref{fig:correctedVSBicubic} shows a visual comparison.

\textbf{Inference run-time.} 
Computing the correction filter requires a negligible amount of time, so it does not change the  run-time of an off-the-shelf DNN.
Using NVIDIA RTX 2080ti GPU, the per image run-time of all the methods except ZSSR is smaller than 1 second (because no training is done in the test phase), while ZSSR requires approximately 2 minutes per image.

\begin{table*}
\footnotesize
\renewcommand{\arraystretch}{1.3}
\caption{Blind super-resolution comparison on Set14. Each cell displays PSNR [dB] (left) and SSIM (right).}
\label{table:super_res_est}
\centering
    \begin{tabular}{ | l || l | l | l | l | l |}
    \hline
         & Scale & Gaussian std = $1.5/\sqrt{2}$ & Gaussian std = $2.5/\sqrt{2}$ & Box of width = 4 \\ \hline 
    KernelGAN & 2 & 26.381 / 0.785 & \textbf{28.868} / \textbf{0.807} & 28.221 / 0.802 \\ \hline 
    DBPN &   2	 & \textbf{30.078} / 0.85 & 26.366 / 0.734 & 28.444 / 0.803 \\ 
    DBPN + our estimated correction & 2 & 28.46 / 0.842 & 28.037 / 0.794 & \textbf{29.778} / \textbf{0.840} \\ \hline

    \end{tabular}
    \begin{tabular}{ | l || l | l | l | l | l |}
    \hline
         & Scale & Gaussian std = $3.5/\sqrt{2}$ & Gaussian std = $4.5/\sqrt{2}$ & Box of width = 8 \\ \hline
    KernelGAN & 4 & 24.424 / 0.673 & 25.174 / 0.669  & 23.575 / 0.634 \\ \hline 
    DBPN & 4 & 25.067 / 0.685 & 23.890 / 0.645 & 24.636 / 0.667 \\ 
    DBPN + our estimated correction & 4 & \textbf{28.184} / \textbf{0.764}  & \textbf{25.542} / \textbf{0.699} & \textbf{25.111} / \textbf{0.681}  \\ \hline

    \end{tabular}
\end{table*}

\begin{table*}
\footnotesize
\renewcommand{\arraystretch}{1.3}
\caption{Blind super-resolution comparison on BSD100. Each cell displays PSNR [dB] (left) and SSIM (right).}
\label{table:super_res_est_bsd100}
\centering
    \begin{tabular}{ | l || l | l | l || l | l | l |}
    \hline
         & Scale & Gaussian std = $1.5/\sqrt{2}$ & Gaussian std = $2.5/\sqrt{2}$ & Scale & Gaussian std = $3.5/\sqrt{2}$ & Gaussian std = $4.5/\sqrt{2}$ \\ \hline
    KernelGAN & 2 & 26.615 / 0.773 & \textbf{28.244} / \textbf{0.780} & 4 & 24.363 / 0.647 & 25.238 / 0.652 \\ \hline        
    DBPN &  2 & \textbf{29.512} / \textbf{0.827} & 26.371 / 0.711 & 4	 & 25.268 / 0.662 & 24.357 / 0.628 \\ 
    DBPN + est. correction & 2 & 27.784 / 0.828 & 27.761 / 0.769 & 4 & \textbf{27.103}/ \textbf{0.722}  & \textbf{25.485} / \textbf{0.671}  \\ \hline
    \end{tabular}  
    
\end{table*}

\subsection{The blind setting}
\label{sec:exp_blind}

In this section, we repeat previous experiments, but without the assumption that the downscaling kernel $\mathbf{k}$ is known.
Therefore, to apply our approach we first estimate the correction filter using Algorithm~\ref{Alg:h_est}, and then use this estimation to restore the HR image by \eqref{Eq_our_SR}.
Note that Algorithm~\ref{Alg:h_est} exploits the pre-trained DNN to estimate the correction filter. Here we apply it only with DBPN, which has a more compact architecture than RCAN and proSR, and hence leads to faster inference. However, similar results can be obtained also for RCAN and proSR.
In this setting we compare our method to kernelGAN \cite{kernelGAN}, which estimates the downscaling kernel using adversarial training (in test-time) and then uses ZSSR to restore the HR image.

\begin{figure}
\captionsetup[subfigure]{labelformat=empty}
    \center
    \begin{subfigure}[b]{0.49\columnwidth}
        \includegraphics[trim={300pt 150pt 0pt 50pt}, clip, width=0.95\textwidth]{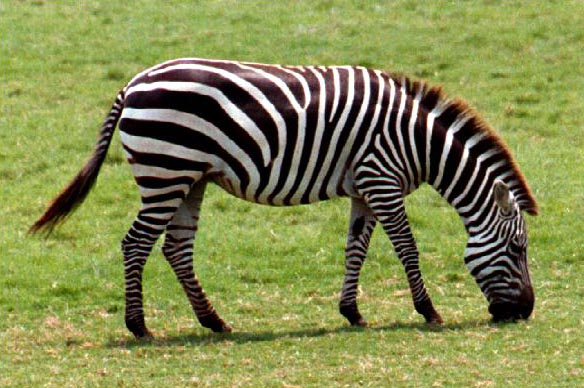}
        \caption{Original image (cropped)}
    \vspace{1mm}
    \end{subfigure}
    \begin{subfigure}[b]{0.49\columnwidth}
        \includegraphics[trim={300pt 150pt 0pt 50pt}, clip, width=0.95\textwidth]{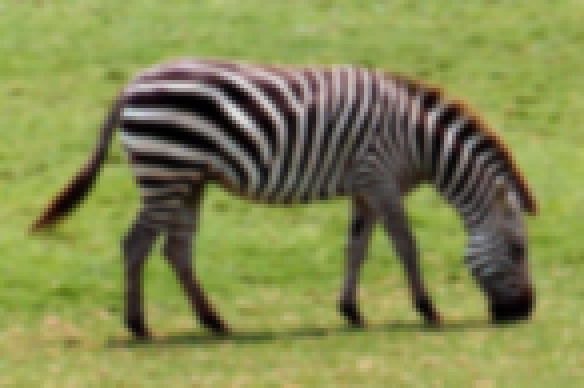}
        \caption{LR image}
    \vspace{1mm}
    \end{subfigure}\\
    \begin{subfigure}[b]{0.49\columnwidth}
        \includegraphics[trim={300pt 150pt 0pt 50pt}, clip, width=0.95\textwidth]{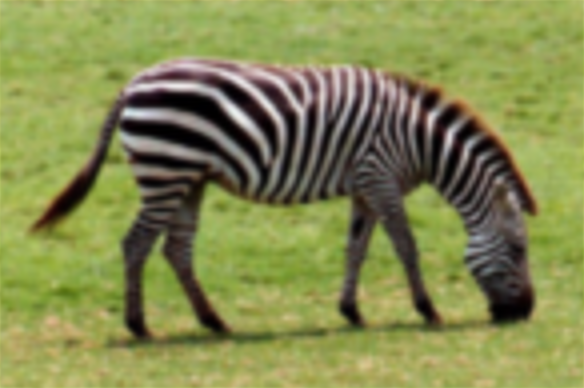}
        \caption{Bicubic upsampling}
    \vspace{1mm}
    \end{subfigure} 
    \begin{subfigure}[b]{0.49\columnwidth}
        \includegraphics[trim={300pt 150pt 0pt 50pt}, clip, width=0.95\textwidth]{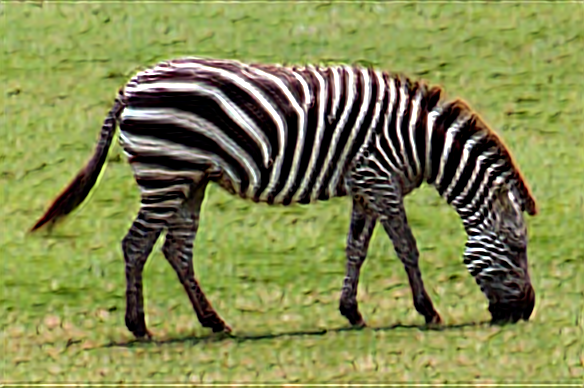}
        \caption{KernelGAN}
    \vspace{1mm}
    \end{subfigure}\\
    \begin{subfigure}[b]{0.49\columnwidth}
        \includegraphics[trim={300pt 150pt 0pt 50pt}, clip, width=0.95\textwidth]{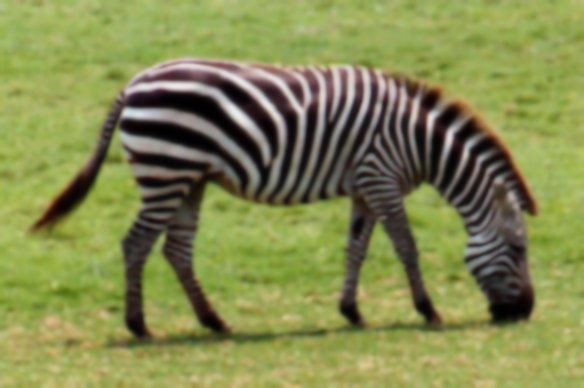}
        \caption{DBPN}
    \vspace{1mm}
    \end{subfigure} 
    \begin{subfigure}[b]{0.49\columnwidth}
        \includegraphics[trim={300pt 150pt 0pt 50pt}, clip, width=0.95\textwidth]{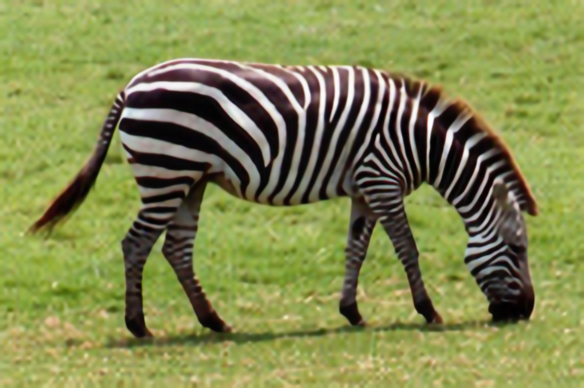}
        \caption{DBPN + est. correction}
    \vspace{1mm}
    \end{subfigure}
    \caption{Blind SR of {\em zebra} image from Set14, for scale factor 4 and Gaussian downscaling kernel with std $3.5/\sqrt{2}$.}
    \label{fig:blind_comparison}
\end{figure}

The results for Set14 and BSD100 are presented in Tables~\ref{table:super_res_est} and \ref{table:super_res_est_bsd100}, respectively, and visual examples are shown 
in Figure~\ref{fig:blind_comparison}. 
More results and comparison on DIV2KRK are presented in the the supplementary material. 
It can be seen that the proposed filter correction approach 
improves the results of DBPN compared to its plain application. 
It also outperforms kernelGAN, despite being much simpler. 

\section{Conclusion}

The SISR task has gained a lot from the developments in deep learning in the recent years. Yet, the leading DNN methods suffer from a huge performance drop when they are tested on images that do not fit the acquisition process assumption used in their training phase---which is, typically, that the downscaling kernel is bicubic.
In this work, we addressed this issue by a signal processing approach: computing a correction filter that modifies the low-resolution observations such that they mimic observations that are obtained with a bicubic kernel. (Notice that our focus in this work on the bicubic kernel is for the sake of simplicity of the presentation and due to its popularity. Yet, it is possible to use our developed tools also for other reconstruction kernels).
The modified LR is then fed into existing state-of-the-art DNNs that are trained only under the assumption of bicubic kernel. 
Various experiments have shown that the proposed approach significantly improves the performance of the pre-trained DNNs and outperforms other (much more sophisticated) methods that are specifically designed to be robust to different kernels.

{\bf Acknowledgment.} 
The work is supported by the NSF-BSF (No. 2017729) and ERC-StG (No. 757497) grants.

{\small
\bibliographystyle{ieee_fullname}
\bibliography{egbib_tom}

\begin{thebibliography}{10}
\bibitem{whittaker1915xviii}
Edmund~Taylor Whittaker.
\newblock On the functions which are represented by the expansions of the
  interpolation-theory.
\newblock {\em Proceedings of the Royal Society of Edinburgh}, 35:181--194,
  1915.
\end{thebibliography}
}

\end{document}